\documentclass[sigconf]{acmart}
\usepackage{amsmath,amssymb}
\usepackage{float}
\usepackage{epsfig}
\usepackage{graphics}
\usepackage{graphicx}
\usepackage{epstopdf}
\usepackage{subcaption}

\usepackage{booktabs} 
\usepackage{algorithm}
\usepackage{algpseudocode}
\newtheorem{theorem}{Theorem}[]

\newtheorem{lemma}[]{Lemma}
\newtheorem{proposition}{Proposition}[]
\newtheorem{assumption}{Assumption}
\newtheorem{remark}{Remark}

\acmConference[VALUETOOLS'19]{ACM VALUETOOLS Conference}{March 2019}{Spain}
\acmYear{2019}

\begin{document}
\title{A Structure-aware Online Learning Algorithm for Markov Decision Processes}
%
%
%

\author{Arghyadip Roy}
\affiliation{%
  \institution{Dept. of Electrical Engineering, IIT Bombay}
  \city{}
}
\email{arghyadip@ee.iitb.ac.in}

\author{Vivek Borkar}
\affiliation{%
  \institution{Dept. of Electrical Engineering, IIT Bombay}
  \city{}
}
\email{borkar@ee.iitb.ac.in}

\author{Abhay Karandikar}
\affiliation{%
  \institution{Dept. of Electrical Engineering, IIT Bombay}
  \institution{Director and Professor, IIT Kanpur}
  \city{}
}
\email{karandi@ee.iitb.ac.in,karandi@iitk.ac.in}

\author{Prasanna Chaporkar}
\affiliation{%
  \institution{Dept. of Electrical Engineering, IIT Bombay}
}
\email{chaporkar@ee.iitb.ac.in}

\renewcommand{\shortauthors}{A. Roy et al.}

\begin{abstract}
To overcome the \textit{curse of dimensionality} and \textit{curse of modeling} in Dynamic Programming (DP) methods
for solving classical Markov Decision Process (MDP) problems,
Reinforcement Learning (RL) algorithms are popular. 
In this paper, we consider an infinite-horizon average reward MDP problem and prove the optimality of the threshold policy under certain
conditions.
Traditional RL techniques do not exploit the threshold nature of optimal policy while learning. In this paper, we propose
a new RL algorithm which utilizes the known threshold structure of the optimal policy while learning by reducing the feasible
policy space. We establish that the proposed algorithm converges to the optimal policy.
It provides a significant improvement in convergence speed and computational and storage complexity over traditional RL algorithms. 
The proposed technique can be applied to a wide variety of optimization problems
that include energy efficient data transmission and management of queues.
We exhibit the improvement in
convergence speed of the proposed algorithm over other RL algorithms through simulations.
\end{abstract}

%
%
%


\maketitle

\section{Introduction}
 The framework of Markov Decision Process (MDP) \cite{puterman2014markov} is used in modeling and optimization of stochastic systems that involve decision making. An MDP is a controlled stochastic process on a state space with an associated control process of `actions', where the transition from one state to the next depends only on the current state-action pair and not on the past history of the system (known as the controlled Markov property).  Each state transition is associated with a reward. Our MDP problem aims to maximize the average reward and provides an optimal \textit{policy} as a solution. A policy is a mapping from a state to an action describing which action is to be
 chosen in a state. An optimal policy maximizes the average reward.\par
 A common approach for solving MDP problems is Dynamic Programming (DP) \cite{puterman2014markov}.
 In this paper, we consider an MDP problem and prove that the optimal policy has a threshold
 structure using DP methods. In other words, we prove that up to a certain threshold in the state space, a specific action is preferred and thereafter
 another action is preferred.\par
 Classical iterative methods for DP are computationally inefficient in the face of
 large state and action spaces.
 This is known as the curse of dimensionality.
 Moreover, the computation of optimal policy using DP methods
 requires the knowledge of state transition probability matrix which is often governed by the statistics of unknown system dynamics. For example,
 in a telecommunication system, transition probabilities between different states are determined by the statistics of arrival
 rates of users. This is known as the curse of modeling.
 In practice, it may be difficult
 to gather the knowledge regarding the statistics of the system dynamics beforehand. When we do not have any prior knowledge of the statistics
 of the system dynamics, a popular approach is Reinforcement Learning (RL) techniques which learn the optimal policy iteratively by trial and error
 \cite{sutton1998reinforcement}.
 Examples of RL techniques include TD($\lambda$) \cite{sutton1998reinforcement}, Q-learning \cite{watkins1992q}, actor-critic \cite{borkar2005actor},
 policy gradient \cite{sutton2000policy} and Post-Decision State (PDS) learning \cite{powell2007approximate,salodkar2008line}.
 Consider, e.g., Q-learning and PDS learning.
 Q-learning \cite{watkins1992q}
 is one of most popular learning algorithms. Q-learning
 iteratively computes the Q-function associated with every state-action pair using
 a combination of exploration and exploitation.
 Since Q-learning needs to learn the optimal policy for all state-action pairs, the storage complexity of
 the scheme is of the order of the cardinality of the state space times the cardinality of the action space. In many cases of practical interest, 
 the state and action spaces are large which renders Q-learning impractical. Furthermore, due to the presence of exploration, the convergence rate of Q-learning is generally slow.
 The idea of PDS \cite{powell2007approximate,salodkar2008line} learning obtained by reformulating the Relative Value Iteration
 Algorithm (RVIA) \cite{puterman2014markov} equation is adopted in literature for various problems. The main advantage of PDS learning
 is that it circumvents the action exploration, thereby improving the convergence rate. Also, there is no need of
 storing the Q functions of state-action pairs. Instead, it requires only storing the value functions associated with the states. Therefore the storage complexity
 of the PDS learning scheme is lower than that of Q-learning.\par
 A common drawback of the learning schemes described above is that they do not exploit any known properties related to the structure of the optimal policy.
 In other words, while learning the optimal policy, the schemes search the optimal policy from the set of all possible policies. However,
 depending on the structure of the optimal policy, the size of feasible action set in various states can be reduced. Moreover, depending on the optimal policy, some
 of the states may not be visited at all. If we incorporate such knowledge in the learning process, intuitively, faster convergence can be achieved
 due to reductions in the state and action spaces or the range of possible policies.
 Furthermore, this may result in reductions in storage and computational complexity
 as well. \par
  In this paper, we propose a Structure-Aware
 Learning (SAL) algorithm which exploits the threshold nature of optimal policy and searches the optimal policy only from the set of threshold
 policies. To be precise, instead of learning the optimal policy for the entire state space, it only learns the threshold in the state space where
 the optimal action changes. Based on the gradient of the average reward of the system, the threshold is updated on a slower timescale than
 that of the value function iterates. As a result, the convergence time of the proposed algorithm reduces along with a reduction in computational
 complexity and storage complexity in comparison to traditional schemes such as Q-learning and PDS learning. We prove that the proposed scheme
 indeed converges to the optimal policy. In general, the proposed technique is applicable to a large variety of optimization problems
 where the optimal policy is threshold in nature, e.g., \cite{agarwal2008structural, sinha2012optimal,koole1998structural,brouns2006optimal,
 ngo2009optimality}.
 Simulation results are presented where the proposed technique is employed on a well-known problem
 from queuing theory \cite{koole1998structural} to demonstrate that the proposed algorithm indeed offers
 faster convergence than traditional algorithms.\par
There are a few works in the literature
 \cite{kunnumkal2008exploiting,fu2012structure,ngo2010monotonicity} which exploit the
 structural properties in the learning framework. In \cite{fu2012structure}, an online learning algorithm which approximates the value functions
 using piecewise linear functions is proposed. However, there is an associated trade-off between complexity and approximation accuracy in this scheme.
 In \cite{kunnumkal2008exploiting}, authors propose a variant of Q-learning where the value function iterates are projected in such a manner that they
 preserve the monotonicity in system state. Similar model is adopted in \cite{ngo2010monotonicity}. Although there is an improvement
 in convergence rate over conventional Q-learning, not much gain in computational complexity is achieved.
 Unlike us, none of these works consider the threshold as a parameter in the learning framework. Therefore they are computationally
 less efficient than our solution.\par
 The rest of the paper is organized as follows. The system model and problem formulation are described in Section \ref{sec:sysmod}. In Section
 \ref{sec:struc}, the optimality of threshold policy is established. In Section \ref{sec:strucRL}, the structure-aware learning algorithm
 is proposed along with a proof of convergence. We provide a comparative study of computational and storage complexities of different RL
 schemes in Section \ref{sec:complexity}. Simulation results are provided in Section \ref{sec:simu}. Section \ref{sec:ext} discusses possible extensions of
 the problem, followed by conclusions in
 Section \ref{sec:conclusion}.
 \section{System Model \& Problem Formulation}\label{sec:sysmod}
 We consider a controlled time-homogeneous Discrete Time Markov Chain (DTMC) and denote it by $\{X_n\}_{n\ge 0},$ which takes values from the finite 
 state space $\mathcal{S}$. Without loss of generality,
 we assume that $\mathcal{S}=\{0,1,2,\ldots,N\}$, where $N$ is a fixed positive integer.
 For the sake of simplicity, we assume that
 each state $i\in \mathcal{S}$ is associated with an action space $\mathcal{A}$. Let the action space $\mathcal{A}$ consists of two actions, viz.,
 $A_1$ and $A_2$. Let the transition probability of going from state $i\in \mathcal{S}$ to state $j\in \mathcal{S}$ under action $a\in \mathcal{A}$
 be denoted as $p_{ij}(a)$. Therefore, we have, $p_{ij}(a) \in [0,1] \ \forall i,j,a$ and $\sum\limits_j p_{ij}(a)=1$. Let the action process
 be denoted by $Z_n,n \ge 0$. Therefore, the evolution of $X_n$ can be described by
 \begin{equation*}
  P(X_{n+1}=j|X_m,Z_m,m\le n,X_n=i)=p_{ij}(Z_n),n\ge 0.
 \end{equation*}
Let us assume that whenever $A_1$ is chosen in state $i\in \mathcal{S}$, no reward is obtained, and the system remains in the same state with
probability $p$ and
goes to state $(i-1)^+$ with the remaining probability,
where $(i)^+=\max\{i,0\}$.
We further assume that whenever the system is in state $i\in \mathcal{S}$ and $A_2$ is chosen, a non-negative fixed reward $r$ is obtained and the system moves to
state $(i+1)$ with probability $p$ and moves to state $(i-1)^+$ with the remaining probability.
Note that the $A_2$ is not feasible in state $N$.\par
We have used this model for sake of specificity and because it does arise in practice. Analogous schemes can be developed for other models that naturally lead to a threshold structure.\par
We aim to obtain a policy which maximizes the average expected reward of the system. Let $\mathcal{Q}$ be the set of memoryless policies where the decision
rule at time $t$ depends only on the state of the system at time $t$ and not on the past history.
Under the assumption of unichain nature of the underlying Markov chain which guarantees the existence of unique stationary distribution, let
the average reward of the system over infinite horizon under policy $Q \in \mathcal{Q}$ be independent of the initial condition and be denoted by $\sigma_Q$.
That is, we intend to maximize
\begin{equation}\label{eq:optimal}
 \sigma_Q= \lim_{H\to \infty}\frac{1}{H}\sum_{h=1}^H {\mathbb{E}}_Q [r(X_h,Z_h)],
\end{equation}
where $r(X_h,Z_h)$ denotes the reward function in state $X_h$ under action $Z_h$, and ${\mathbb{E}}_Q$ denotes the expectation
operator under policy $Q$. The limit in Equation (\ref{eq:optimal}) may be taken to exist because the optimal policy is known to be stationary.
The DP equation depicted below provides the necessary condition for optimality $\forall i \in \mathcal{S}$.
\begin{equation}
 V(i)=\max_{a \in \mathcal{A}}\left[ r(i,a)+\sum_{j \in \mathcal{S}}p_{ij}(a)V(j)-\sigma \right],
\end{equation}
where $V(i)$ and $\sigma$ denote the value function of state $i\in \mathcal{S}$ and the optimal average reward, respectively. The $\arg\max$ above
yields the optimal policy, i.e., optimal action as a function of current state.
RVIA can be used to solve this problem using the iterative scheme described below.
\begin{equation}\label{eq:rvia}
  V_{n+1}(i)=\max_{a \in \mathcal{A}}\left[r(i,a)+\sum_{j \in \mathcal{S}}p_{ij}(a)V_n(j)-V_n(i^*)\right],
\end{equation}
where $V_n(.)$ is the value function estimate in $n^{\rm{th}}$ iteration of RVIA and $i^*\in \mathcal{S}$ is a fixed state.

\section{Structure of Optimal Policy}\label{sec:struc}
In this section, we investigate the structure of the optimal policy. We prove the structural properties using the `non-increasing difference' property
of the value function in the lemma described next.
\begin{lemma}\label{lemma1}
$V(i+1)-V(i)$ is non-increasing in $i$.
\end{lemma}
\begin{proof}
Proof is presented in Appendix \ref{app0}.
\end{proof}

The following theorem describes that the optimal policy is of threshold type where $A_2$ is optimal only upto a certain threshold.
\begin{theorem}
 The optimal policy has a threshold structure where the optimal action changes from $A_2$ to $A_1$ after a certain threshold
 in $i\in \mathcal{S}$.
\end{theorem}
\begin{proof}
 If $A_1$ is optimal in state $i$, then $r+V(i+1)\le V(i)$. 
Using Lemma \ref{lemma1}, $V(i+1)-V(i)$ is non-increasing in $i$.
 Therefore, it follows that there exists a threshold such that $A_2$ is optimal only below the threshold, $A_1$ thereafter.
\end{proof}
\section{Structure-aware Online RL Algorithm}\label{sec:strucRL}
In this section, we propose a learning algorithm by exploiting the threshold properties of the optimal policy. Unlike the traditional RL
algorithms which optimize over the entire policy space, our algorithm searches the optimal policy only from the set of threshold policies.
As a result, the proposed algorithm converges faster than traditional RL algorithms like Q-learning, PDS learning. Also, the
computational complexity and the storage complexity of learning is reduced as argued later.
\subsection{Gradient Based RL Framework}
 Since we know that the optimal policy is threshold in nature where the optimal action changes from $A_2$ to $A_1$ after a certain threshold,
 if we know the value of the threshold,
 we can specify the optimal policy completely. However, the value of the threshold depends on the transition probabilities (i.e., $p$) between different
 states.
 Therefore, in the absence of knowledge regarding $p$, instead of learning the optimal policy from the set of all policies, we only learn the the optimal value
 of the threshold. We target to optimize over the threshold using an update rule so that the value of threshold converges to the optimal threshold.\par
We consider the set of threshold policies and describe them in terms of the value of parameter threshold ($T$, say). The approach we adopt in this paper is to
compute the gradient of the average expected reward of the system with respect to the threshold $T$ and improve the threshold policy in the
direction of the gradient by updating the the value of $T$. Before proceeding, we need to explicitly indicate the dependence of the associated MDP
on $T$ by redefining the notations in the context of threshold policies. \par
Let the steady state stationary probability of state $i$, the value function of state $i$ and the average reward of the Markov chain in terms of threshold
parameter $T$ be denoted by ${\pi}(i,T)$, $V(i,T)$ and $\sigma(T)$, respectively. Let the transition probability from state $i$ to state $j$
under threshold $T$ be denoted as $P_{ij}(T)$.
Therefore,
\begin{equation*}
 P_{ij}(T)=P(X_{n+1}=j|X_n=i,T).
\end{equation*}
We later embed the discrete parameter $T$ into a continuous valued one. With this in mind, we make the following assumption regarding $P_{ij}(T)$.
\begin{assumption}\label{vtools:ass2}
$P_{ij}(T)$ is a twice differentiable function of $T$ with bounded first and second derivatives. Moreover, $P_{ij}(T)$
is bounded.
\end{assumption}
The proposition
described below provides a closed-form expression for the gradient of the average reward $\sigma(T)$.
\begin{proposition}
Under Assumption \ref{vtools:ass2},
\begin{equation*}
 \nabla \sigma(T)=\sum_{i \in \mathcal{S}} {\pi}(i,T) \sum_{j \in \mathcal{S}} \nabla P_{ij}(T)V(j,T).
\end{equation*}

\end{proposition}

\begin{proof}
 Detailed proof can be found in \cite{marbach2001simulation}.
\end{proof}
The system model considered by us is a special case of the model considered in \cite{marbach2001simulation}, with the exception that unlike in
\cite{marbach2001simulation},
the reward function in our case does not have any dependence on $T$. 
\subsection{Online RL Algorithm}
Optimal policy can be obtained using RVIA if the transition probabilities between different states are known beforehand. In the absence of knowledge
regarding transition probabilities, we can use theory of Stochastic Approximation (SA) \cite{borkar2008stochastic} to remove the expectation
operation in Equation (\ref {eq:rvia}) and converge to the optimal policy by averaging over time.
Let $g(n)$ be a positive step-size sequence having the following properties.
\begin{equation}\label{eq:robbbin}
 \sum_{n=1}^{\infty}g(n)=\infty; \sum_{n=1}^{\infty}(g(n))^2<\infty.
\end{equation}
Let $h(n)$ be another step-size sequence with similar properties as in Equation (\ref{eq:robbbin}) along with the following additional property.
\begin{equation}\label{eq:robbbin1}
 \lim_{n \to \infty} \frac{h(n)}{g(n)}\to 0.
\end{equation}
In order to learn the optimal policy, we adopt the following strategy. We update the value function of one state at a time and keep
others unchanged. Let $S_n$ be the state whose value function is updated at $n^{\rm{th}}$ iteration. Let $\eta(i,n)$ denote the number
of times the value function of the state $i$ is updated till $n^{\rm{th}}$ iteration. Symbolically,
\begin{equation*}
 \eta(i,n)=\sum_{m=0}^n I\{i=S_m\}.
\end{equation*}
The scheme for the update of value function can be described as follows.
\begin{equation}\label{primal}
\begin{split}
& V_{n+1}(i,T)=(1-g(\eta(i,n)))V_n(i,T)+g(\eta(i,n))[r(i,a)+V_n(j,T)\\&-V_n(i^*,T)],\\&
V_{n+1}(i',T)=V_{n}(i',T), \forall i'\neq i,
\end{split}
\end{equation}
where $V_{n}(i,T)$ denotes the value function of state $i$ at the $n^{\rm{th}}$ iteration on the faster timescale when the current value of threshold is $T$.
The scheme (\ref{primal}) solves a dynamic programming equation for a fixed value of threshold $T$, referred to as primal RVIA.
To obtain the optimal threshold value, $T$ has to be iterated in a separate timescale $h(n)$.
Intuitively, in order to learn the value of the optimal threshold, we can determine the value of $\nabla \sigma(T)$ based on the current value
of threshold $T_n$ at the $n^{\rm{th}}$ iteration and then update the value of threshold in the direction of the gradient. This is similar to a stochastic
gradient scheme which can be expressed as
\begin{equation}\label{Tupdate}
 T_{n+1}=T_n+h(n)\nabla \sigma(T_n).
\end{equation}
The assumptions described in Equations (\ref{eq:robbbin}) and (\ref{eq:robbbin1}) guarantee that value function and threshold parameter are
updated in two separate timescales without interfering in each other's convergence behavior. The value functions are updated in a faster timescale
than that of the threshold. From the faster timescale, the value of threshold appears to be fixed. From the slower timescale, the value
functions seem to be equilibrated according to the current threshold value. This behavior is commonly known as
``leader-follower'' scheme.\par
Given a threshold $T$, we assume that the transition from state $i$ is determined by the rule $P^1(j|i)$, if $i<T$
and by the rule $P^0(j|i)$, otherwise. For example, consider that the system is in state $i$ and $i<T$. Then the next state to which
the system moves is governed by the rule $P^1(j|i)$ for action $A_2$. Therefore, the system moves to the state $(i+1)$.
However, if $i \ge T$, then the state transition is given by the rule $P^0(j|i)$ for action $A_1$. 
Therefore, the system remains in state $i$. 
This scheme is applied to Equation (\ref{primal}) for a fixed value of threshold $T$.\par
To update the threshold, we need to interpolate the value of threshold which takes discrete values, to continuous domain
so that the online rule can be applied. Since the threshold policy can be described as a step function which takes discrete non-negative values
as input and follows $P^1(j|i)$
upto a threshold and $P^0(j|i)$ thereafter, the derivative does not exist at all points (See Assumption \ref{vtools:ass2}).
Therefore, we propose an approximation to the threshold policy using a randomized policy. The randomized policy is a mixture of
two policies depicted by $P^0(j|i)$ and $P^1(j|i)$ with corresponding probabilities $f(i,T)$ and $(1-f(i,T))$.
To be precise,
\begin{equation}\label{appx}
 P_{ij}(T) \approx P^0(j|i)f(i,T)+P^1(j|i)(1-f(i,T)).
\end{equation}
Note that the function $f(.,.)$ which decides how much importance is to be given to respective policies, is a function of state $i$ and
current value of threshold $T$. For a convenient approximation, $f(i,T)$ should be an increasing function
of $i$. The idea is to provide comparable importances to both $P^0(j|i)$ and $P^1(j|i)$ near the threshold and reduce the importance
of $P^0(.|.)$ ($P^1(.|.)$) away from the threshold in the left (right) direction. We choose the following function
owing to its nice properties such as continuous differentiability and the existence of non-zero derivative everywhere.
\begin{equation}\label{sigmoid}
 f(i,T)=\frac{e^{(i-T-0.5)}}{1+e^{(i-T-0.5)}}.
\end{equation}
This does not satisfy Assumption 1 at $T = i, i -1$, but that does not affect our subsequent analysis if we take right or left derivatives at these points.
\begin{remark}
 Another choice of $f(i,T)$ could be the following.
 \begin{equation*}
f(i,T)=0.I\{i\le T\}+1.I\{i\ge T+1\}+(i-T).I\{T< i< T+1\}.
 \end{equation*}
Since this function exactly replicates the step function nature of the optimal policy in the interval $[0,T]$ and $[T+1,N]$ and uses
approximation only in the interval $(T,T+1)$,
the approximation error in this case is less than that of Equation (\ref{sigmoid}).
However, the derivative of the function is nonzero only in the interval $(T,T+1)$.
Therefore, if the initial guess of the threshold is outside this range, then the proposed learning scheme may not converge to the optimal
threshold as the gradient becomes zero.
\end{remark}
While devising an update rule for the threshold, we evaluate $\nabla P_{ij}(T)$ as a representative of $\nabla \sigma(T)$ and use
that in Equation (\ref{Tupdate}). From Equation (\ref{appx}), we get,
\begin{equation}\label{nabla}
\nabla P_{ij}(T) = (P^0(j|i)-P^1(j|i))\nabla f(i,T).
\end{equation}
Since multiplication by a constant factor does not impact the online update of the proposed scheme,
we incorporate an extra multiplicative factor of $\frac{1}{2}$ to the right hand side of Equation (\ref{nabla}). This operation
can be described in the following manner. In every iteration, we choose transition according to $P^0(.|.)$ and $P^1(.|.)$
with equal probabilities. $\nabla f(i,T)$ is a state-dependent term which denotes how much importance
is to be given to the value function of the state. Therefore, the update of
$T$ in the slower timescale $h(n)$ is as follows.
\begin{equation*}
 T_{n+1}=\Lambda[T_n+h(n)\nabla f(i,T_n){(-1)}^\gamma V_n(k,T_n)],
\end{equation*}
where $\gamma$ is a random variable which takes values $0$ and $1$ with equal probabilities.
If $\gamma=0$, then the transition is determined by the rule $P^0(.|.)$, else by $P^1(.|.)$.
Therefore, $k\sim \tilde{P}_{ik}$ where $\tilde{P}_{ik}=\gamma P^0(k|i)+ (1-\gamma) P^1(k|i)$.
The averaging effect of SA scheme enables us to obtain the effective drift in Equation (\ref{nabla}).
The projection operator $\Lambda$ is introduced to guarantee that the iterates remain bounded in $[0,N]$.\par
Therefore, the online RL scheme where the value functions are updated in the faster timescale and the threshold parameter
in the slower one, can be summarized as
\begin{equation}\label{primal1}
\begin{split}
& V_{n+1}(i,T)=(1-g(\eta(i,n)))V_n(i,T)+g(\eta(i,n))[r(i,a)+V_n(j,T)\\&-V_n(i^*,T)];\\&
V_{n+1}(i',T)=V_{n}(i',T), \forall i'\neq i,
\end{split}
\end{equation}

\begin{equation}\label{dual1}
 T_{n+1}=\Lambda[T_n+h(n)\nabla f(i,T_n){(-1)}^\gamma V_n(k,T_n)].
\end{equation}
The transitions in (\ref{primal1}) from $i$ to $j$ correspond to a single run of a simulated chain as is common in RL. For each current state $i$, the $k$ in (\ref{dual1}) is generated separately as per $\tilde{P}_{ik}$.

\begin{theorem}\label{theo1}
The schemes (\ref{primal1}) and (\ref{dual1}) converge to optimality almost surely (a.s.).
\end{theorem}
\begin{proof}
 Proof is provided in Appendix \ref{appa}.
\end{proof}
We describe the resulting two-timescale SAL algorithm in Algorithm \ref{algo} .
\begin{algorithm}
\caption{Two-timescale SAL algorithm}\label{algo}
\label{NCalgorithm}
\begin{algorithmic}[1]

\State Initialize number of iterations $n \leftarrow 1$,
 value function ${V}({i}) \leftarrow 0, \forall {i}\in \mathcal{S}$ and the threshold $T \leftarrow 0$.
\While {TRUE}
\State Choose action $a$ governed by the current value of $T$. 

\State Update the value function of state ${i}$ using Equation (\ref{primal1}).
\State Update threshold $T$ using Equation (\ref{dual1}).
\State Update ${i}\leftarrow {j}$ and $n\leftarrow n+1$.
\EndWhile
\end{algorithmic}
\end{algorithm}
As described in Algorithm \ref{algo}, the number of iterations, value functions and the threshold are initialized at the beginning.
On every decision epoch, we choose the action which is specified by the current value of threshold.
Based on the reward obtained, the value function of states and the value of threshold are updated in faster and slower timescale, respectively.
The rules for the updates are provided in Equation (\ref{primal1}) and (\ref{dual1}), respectively.
\begin{remark}
Even if the optimal policy in an MDP problem does not have a threshold structure, the methodologies presented in this paper
which is guaranteed to converge to the optimal (at least locally) threshold policy, can be used. In general threshold policies are
easy to implement and have low storage complexity.
Besides, often a well chosen threshold policy provides a good performance.
\end{remark}

\section{Computational and Storage Complexity}\label{sec:complexity}
In this section, we provide a comparative study of computational and storage complexities associated with traditional learning algorithms
such as Q-learning, PDS learning and the SAL algorithm. The comparison is summarized in Table \ref{table}.\par
\begin{table}[ht]
\caption{Computational and storage complexities of various RL algorithms.}\label{table}
\centering
\begin{tabular}{|l||l||l|}
\hline
\textbf{Algorithm} & \textbf{Storage} & \textbf{Computational}\\
\textbf{} & \textbf{complexity} & \textbf{complexity}\\ \hline
Q-learning \cite{sutton1998reinforcement,watkins1992q} & $O(|\mathcal{S}|\times|\mathcal{A}|)$ & $O(|\mathcal{A}|)$\\ \hline
PDS learning \cite{salodkar2008line,powell2007approximate}& $O(|\mathcal{S}|)$ &  $O(|\mathcal{A}|)$  \\ \hline
SAL & $O(|\mathcal{S}|)$ & $O(1)$\\ \hline
\end{tabular}
\end{table}
As described in Table \ref{table}, Q-learning algorithm needs to store the value function associated with every state-action pair. Thus, the storage
complexity associated with Q-learning is $O(|\mathcal{S}|\times|\mathcal{A}|)$. PDS learning algorithm needs to store
the value functions associated with only the PDSs along with feasible actions in every state, thereby requiring $O(|\mathcal{S}|)$ storage.
The SAL algorithm
proposed by us needs to store the value functions of all the states and the value of threshold. We no longer need to store
feasible actions corresponding to every state since the value of threshold completely specifies the policy. Therefore, the storage complexity of SAL
algorithm is $O(|\mathcal{S}|)$.
However, for all practical purposes, once the algorithm converges, it is sufficient to store only the value of threshold instead of optimal
actions associated with every state, as required by Q-learning and PDS learning.\par
Q-learning algorithm updates the value function associated with a state-action pair in every iteration by evaluating
$|\mathcal{A}|$ functions and choosing the best one. Therefore, the per-iteration complexity associated with Q-learning is
$O(|\mathcal{A}|)$. In the case of PDS learning, each iteration involves the evaluation of $|\mathcal{A}|$ functions, thereby having a
per-iteration complexity of $O(|\mathcal{A}|)$. As evident form Equation (\ref{primal1}) and (\ref{dual1}), single iteration of the proposed
algorithm involves updating the value function of a state and the value of threshold. Therefore, the computational complexity
of our proposed algorithm is $O(1)$. This is a considerable reduction in computational complexity in comparison to Q-learning
and PDS learning.
\section{Simulation Results}\label{sec:simu}
In this section, we demonstrate the advantages offered by the proposed algorithm in terms of convergence speed with respect to other traditional algorithms such as
Q-learning \cite{sutton1998reinforcement}, PDS learning \cite{salodkar2008line}. We adopt a simple queuing model from \cite{koole1998structural}
and exhibit that the SAL algorithm converges faster than other RL algorithms.
In general, the proposed learning technique is applicable to models involving threshold structure of the optimal policy,
such as \cite{agarwal2008structural,sinha2012optimal,brouns2006optimal,ngo2009optimality}.\par
Authors in \cite{koole1998structural}
consider a single queue where the service time is exponentially distributed (with parameter $\frac{1}{\mu}$, say), and the arrival process is Poisson. The system incurs a constant cost upon
blocking a user. Additionally, there is a holding cost which is a convex function of the number of customers in the system. Authors prove that
it is optimal to admit a user only below a threshold on the number of customers.
We conduct ns-3 simulations of SAL
algorithm to exploit the threshold structure of optimal policy in \cite{koole1998structural}
and compare the convergence performance with Q-learning and PDS learning
algorithms.
\subsection{Convergence Analysis}
\begin{figure}
    \centering
    \begin{subfigure}[b]{0.34\textwidth}
       \includegraphics[width=\textwidth]{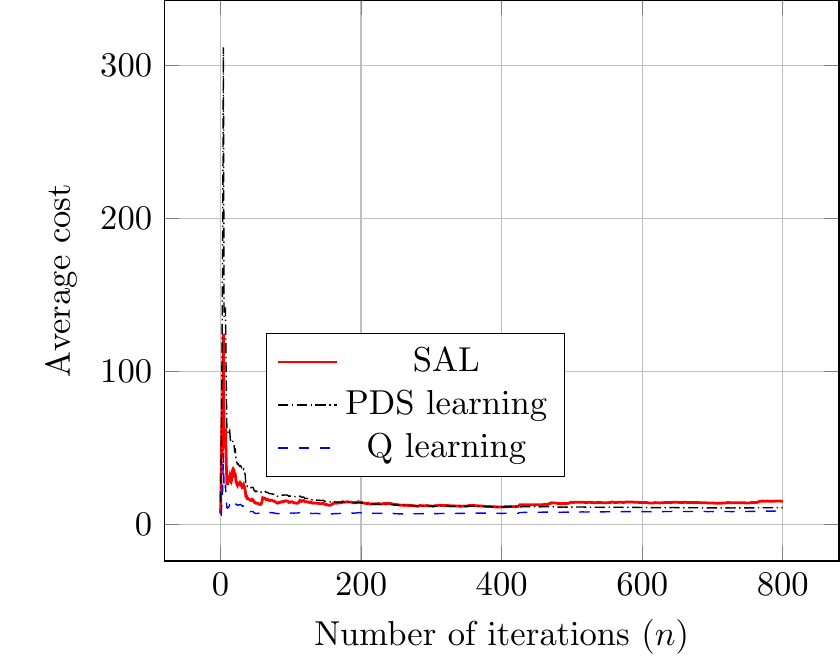}
        \caption{{$\mu=1.2 s^{-1}$.}}
        \label{fig:cost_1}
        \end{subfigure}%
    \\
    \begin{subfigure}[b]{0.34\textwidth}
       \includegraphics[width=\textwidth]{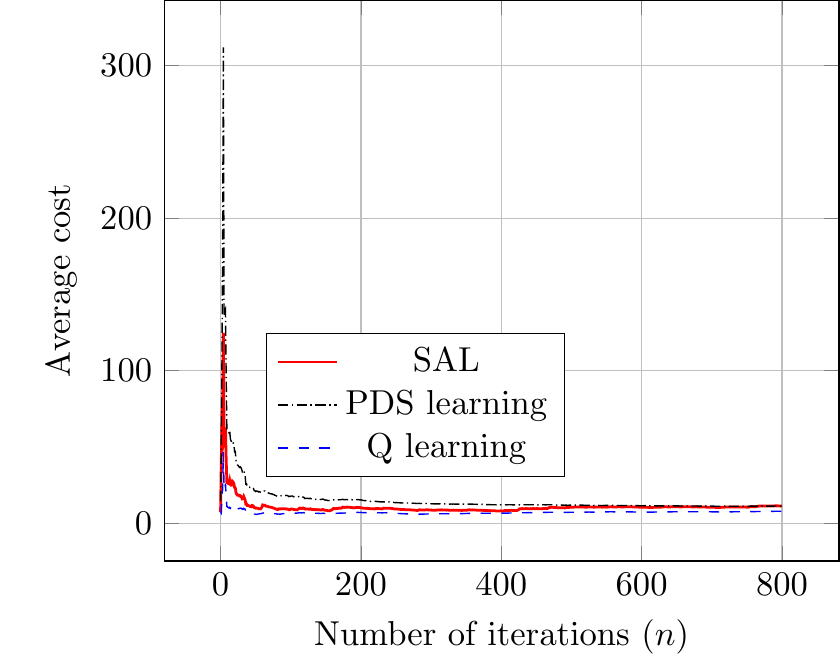}
        \caption{{ $\mu=1.5 s^{-1}$.}}
        \label{fig:cost_2}
        \end{subfigure}

\caption{{Plot of average cost vs. number of iterations $(n)$ for different algorithms.}}
\end{figure}
As illustrated in Fig. (\ref{fig:cost_1}) and (\ref{fig:cost_2}), SAL algorithm converges faster than both Q-learning and PDS learning.
Due to the absence of exploration mechanism, PDS learning has better convergence behavior than Q-learning. However, SAL algorithm
outperforms both Q-learning and PDS learning due to the fact that it operates on a smaller feasible policy space (set of threshold policies only)
than other algorithms. On the other hand, for both Q-learning and PDS learning, the policy at any given iteration may be non-threshold
in nature. This increases the convergence time to optimality. As observed in Fig. \ref{fig:cost_1},
while Q-learning and PDS learning require around $600$  and $200$ iterations, respectively, for convergence, SAL algorithm requires only
$50$ iterations. Similarly, in Fig. \ref{fig:cost_2}, the number of iterations reduces from $750$ in Q-learning and $500$ in PDS learning
to $100$ iterations in SAL algorithm.\par

\begin{figure}
    \centering
    \begin{subfigure}[b]{0.34\textwidth}
       \includegraphics[width=\textwidth]{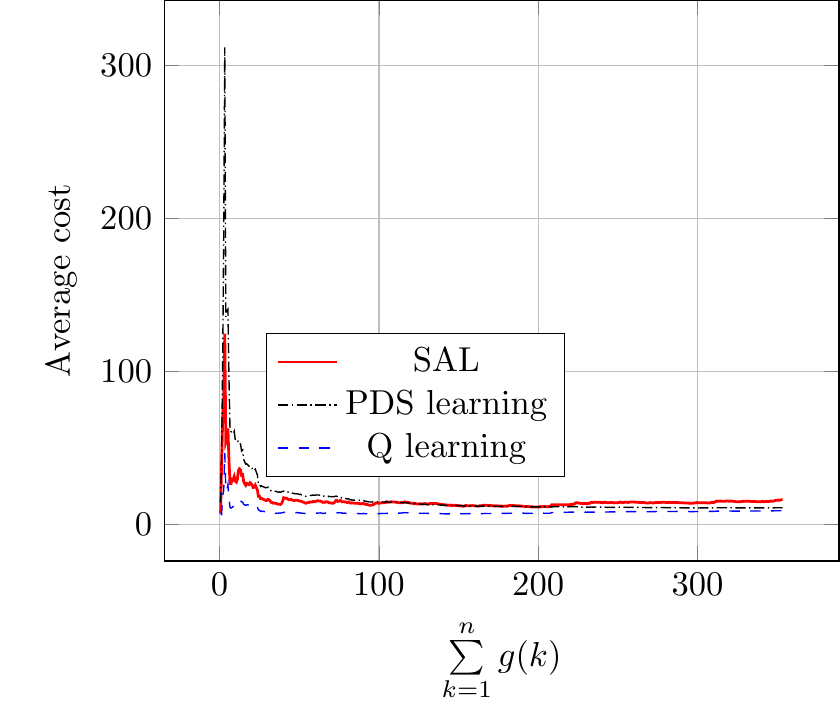}
        \caption{{$\mu=1.2 s^{-1}$.}}
        \label{fig:cost_step_1}
        \end{subfigure}%
    \\
    \begin{subfigure}[b]{0.34\textwidth}
       \includegraphics[width=\textwidth]{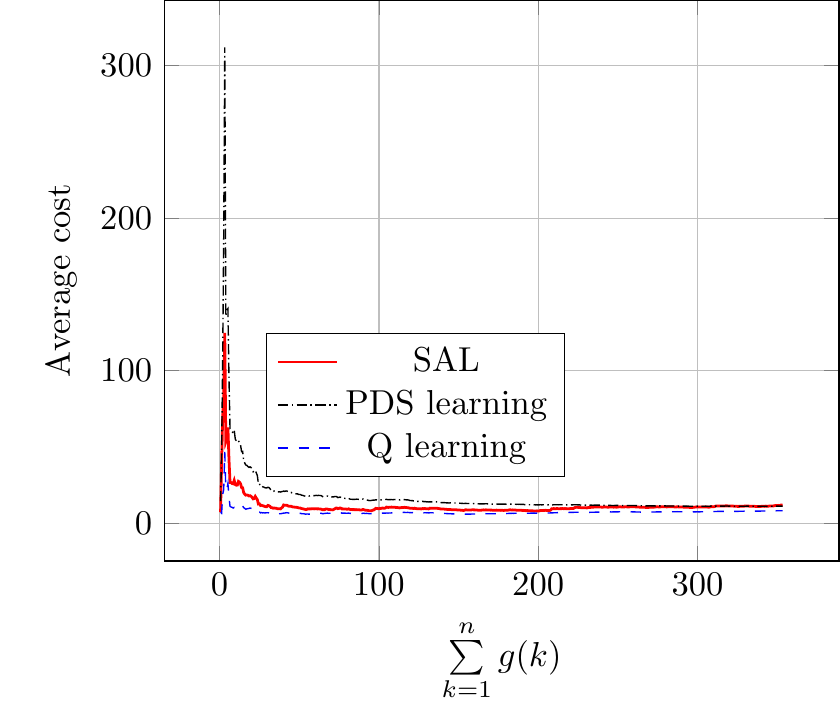}
        \caption{{ $\mu=1.5 s^{-1}$.}}
        \label{fig:cost_step_2}
        \end{subfigure}

\caption{{Plot of average cost vs. sum of step sizes till $n^{\rm {th}}$ iteration for different algorithms.}}
\end{figure}
However, for practical purposes, even if we do not converge to the optimal policy, if the average cost of the system
does not change much over a window of iterations, we can say that the stopping criterion is reached. In other words,
the current policy is close to the optimal policy with high probability. Instead of a window of iterations,
we consider the sum of step sizes till the present iteration as the parameter of choice to eliminate the effect of
declining step size in convergence. We choose the window size equal to $50$ and observe
in Fig. \ref{fig:cost_step_1} that convergence for Q-learning, PDS learning and SAL algorithm are achieved
approximately in $500,175$ and $50$ iterations. Similarly, we observe in Fig. \ref{fig:cost_step_2} that number of iterations
required for practical convergence reduces from $700$ and $300$ in Q-learning and PDS learning to $75$
in SAL algorithm.
\section{Possible Extensions}\label{sec:ext}
In this section, we describe the possible extensions of the techniques proposed in this paper.
Although the techniques employed in this paper are primarily focused towards solving MDP problems,
the techniques can be employed for learning problems involving Constrained MDP (CMDP) problem also.
Due to the presence of constraints, usually a two-timescale learning
approach is adopted \cite{borkar2008stochastic}, where the value functions are updated in one timescale and
the associated Lagrange Multiplier (LM) in another. Consideration of structure-aware
learning may introduce another timescale where the value of threshold is updated. However,
since the iterates for the  LM and the threshold are not dependent on each other,  they can be updated in the same timescale.\par
The proposed learning technique can also be extended to MDP/ CMDP problems parameterized by a set of threshold parameters rather than only one. .
In the slower timescale, one threshold parameter can be updated in a single
iteration based on the visited state and rest can be kept fixed. 
Since the update of threshold parameters follows a stochastic gradient scheme,
contrary to value function iterates, the threshold parameter iterates do not need individual local clocks for convergence.
However, for the scheme to work, the relative frequencies of the update of individual threshold parameters have to be bounded away
from zero \cite{borkar2008stochastic}. Yet another future direction is to develop RL schemes for restless bandits wherein
threshold policies often lead to simple index-based policies, see \cite{borkar2018reinforcement} for a step towards this.
\section{Conclusions}\label{sec:conclusion}
In this paper, we have considered an MDP problem and proved the optimality of threshold policies.
To this end, we have proposed a RL algorithm which exploits the threshold structure of the optimal policy while learning.
Contrary to traditional RL algorithms, the proposed algorithm searches the optimal policy only from the set of threshold
policies and hence provides faster convergence. We have proved that the proposed scheme indeed converges to
the globally optimal threshold policy. Analysis has been presented to exhibit the effectiveness of the proposed
technique in reducing the computational and storage complexity. Simulation results demonstrate the
improvement in convergence behavior of the proposed algorithm in comparison to that of Q-learning and PDS learning.
\appendix
\section{Proof of Lemma \ref{lemma1}}\label{app0}
We rewrite the optimality equation for the value function as
 \begin{equation*}
V(i)=p\max\{V(i),r+V(i+1)\}+(1-p)V((i-1)^{+}).
 \end{equation*}
Let the value function of state $i$ in $n^{\rm th}$ iteration of Value Iteration Algorithm (VIA) be denoted by $v_n(i)$.
Start with $v_0(i)=0$. 
Hence, $v_0(i+1)-v_0(i)$ is non-increasing in $i$.
We have,
 \begin{equation}\label{eq:vit}
v_{n+1}(i)=p\max\{v_n(i),r+v_n(i+1)\}+(1-p)v_n((i-1)^{+}).
 \end{equation}

Using Equation (\ref{eq:vit}), 
$v_1(i+1)-v_1(i)$ is non-increasing in $i$.
Now, we assume that $v_n(i+1)-v_n(i)$ is non-increasing in $i$. We need to prove that $v_{n+1}(i+1)-v_{n+1}(i)$ is non-increasing in $i$.
 Let us define $v'_{n+1}(i,a)$ as follows.
 \begin{equation*}
v'_{n+1}(i,a)=
 \begin{cases}
      v_n(i) ,&a=A_1,\\
      r+v_n(i+1),&a=A_2.\\
\end{cases}
\end{equation*}
Also define $v'_{n}(i)=\max\limits_{a\in \mathcal{A}} v'_{n}(i,a)$.
Let us define $D v_n(i)=v_n(i+1)-v_n(i)$.
Therefore,
\begin{equation*}
D v'_{n+1}(i,a)=
 \begin{cases}
         D v_{n}(i),&a=A_1,\\
        D v_n(i+1),&a=A_2.\\
\end{cases}
\end{equation*}

\begin{equation*}
D^2 v'_{n+1}(i,a)=
 \begin{cases}
        D^2 v_{n}(i),&a=A_1,\\
        D^2 v_n(i+1),&a=A_2.\\
\end{cases}
\end{equation*}

Since $v_n(i+1)-v_n(i)$ is non-increasing in $i$, $v'_{n+1}(i+1,a)-v'_{n+1}(i,a)$ is non-increasing in $i,$ $\forall a \in \mathcal{A}$.
Let $a_1 \in \mathcal{A}$ and $a_2 \in \mathcal{A}$ be the maximizing actions in states $(i+2)$ and $i$, respectively.

\begin{equation*}
 \begin{split}
  & 2v'_{n+1}(i+1)\ge v'_{n+1}(i+1,a_1)+v'_{n+1}(i+1,a_2)\\&
  =v'_{n+1}(i+2,a_1)+v'_{n+1}(i,a_2)+D v'_{n+1}(i,a_2)-D v'_{n+1}(i+1,a_1).
  \end{split}
 \end{equation*}
 Let $B=D v'_{n+1}(i,a_2)-D v'_{n+1}(i+1,a_1)$.
For proving that $v'_{n+1}(i+1)-v'_{n+1}(i)$ is non-increasing in $i$, we need to prove $B\ge 0$.
Let us consider four cases as follows.
\begin{itemize}
 \item $a_1=a_2=A_1$
 \begin{equation*}
 \begin{split}
 B= D v_{n}(i)-D v_{n}(i+1)=-D^2 v_n(i) \ge 0.
  \end{split}
 \end{equation*}
\item $a_1=A_1,a_2=A_2$
 \begin{equation*}
 \begin{split}
 B= D v_{n}(i+1)-D v_{n}(i+1)= 0.
  \end{split}
 \end{equation*}
\item $a_1=a_2=A_2$
 \begin{equation*}
 \begin{split}
 &B= D v_{n}(i+1)-D v_n(i+2)
 =-D^2 v_{n}(i+1) \ge 0.
  \end{split}
 \end{equation*}
 \item $a_1=A_2, a_2=A_1$
 \begin{equation*}
 \begin{split}
 &B= D v_{n}(i)-D v_n(i+2)
 =-D^2 v_{n}(i)-D^2 v_{n}(i+1) \ge 0.
  \end{split}
 \end{equation*}

\end{itemize}
Since $v'_{n+1}(i+1)-v'_{n+1}(i)$ and $v_n(i+1)-v_n(i)$ are non-increasing in $i$, $v_{n+1}(i+1)-v_{n+1}(i)$ is non-increasing in $i$ (Using (\ref{eq:vit})).
Since $V(.)=\lim \limits_{n \to \infty}v_n(.)$, $V(i+1)-V(i)$ is non-increasing in $i$.
\section{Proof of Theorem \ref{theo1}}\label{appa}
Proof methodologies adopted in this paper are similar to that of \cite{salodkar2008line}. The idea of adoption of Ordinary Differential Equation
(ODE) approach for analyzing SA algorithms by considering them as a noisy discretization of a limiting ODE \cite{borkar2008stochastic},
is considered. Step size parameters are considered as discrete time steps, and if the discrete values of the iterates are linearly interpolated,
they closely follow the trajectory of the ODE. Assumptions on step sizes, viz., (\ref{eq:robbbin}) and (\ref{eq:robbbin1}) are made to
guarantee that the discretization error and error due to noise are negligible asymptotically. As a result, in the asymptotic sense,
the iterates closely follow the trajectory of the ODEs and converge a.s. to the globally asymptotically stable equilibrium.\par
Update rules for value functions and threshold in the faster and slower timescale, respectively, are as follows.
\begin{equation}\label{primal2}
\begin{split}
& V_{n+1}(i,T)=(1-g(\eta(i,n)))V_n(i,T)+g(\eta(i,n))[r(i,a)+V_n(j,T)\\&-V_n(i^*,T)];\\&
V_{n+1}(i',T)=V_{n}(i',T), \forall i'\neq i,
\end{split}
\end{equation}

\begin{equation}\label{dual2}
 T_{n+1}=\Lambda[T_n+h(n)\nabla f(i,T_n){(-1)}^\gamma V_n(k,T_n)].
\end{equation}

Following the two timescale analysis adopted in \cite{borkar2008stochastic}, we consider Equation (\ref{primal2}) first keeping threshold
$T$ fixed.
Let $M_1:\mathbb{R}^{|\mathcal{S}|} \to \mathbb{R}^{|\mathcal{S}|}$ be a map given by
\begin{equation}\label{nonzero}
M_1(s)=\sum_j P_{ij}(T)[r(i,a)+V_n(j,T)]-V_n(i^*,T).
\end{equation}
The knowledge of $P_{ij}(T)$ is required only for the sake of analysis. However, the proposed algorithm can operate without the knowledge
of $P_{ij}(T)$. Since $T$ is kept constant, this gives rise to the following limiting ODE which tracks Equation (\ref{primal2}).
\begin{equation}\label{original}
\dot{V}(t)=M_1(V(t))-V(t).
\end{equation}
As $t\to \infty$, $V(t)$ converges to the fixed point of $M_1(.)$ (i.e., $M_1(V)=V$)\cite{konda1999actor},
which is the asymptotically stable equilibrium of the ODE.
Similar approaches are adopted in \cite{abounadi2001learning,konda1999actor}.\par
The lemma presented next establishes the boundedness of value functions and threshold iterates.
\begin{lemma}
 The value function and the threshold iterates are bounded a.s.
\end{lemma}
\begin{proof}
Let $M_0:\mathbb{R}^{|\mathcal{S}|} \to \mathbb{R}^{|\mathcal{S}|}$ be a map given by
\begin{equation}\label{zero}
M_0(s)=\sum_j P_{ij}(T)V_n(j,T)-V_n(i^*,T).
\end{equation}
Note that Equation (\ref{nonzero}) reduces to (\ref{zero}) if the immediate reward is zero.
Now, $\lim \limits_{b\to \infty} \frac{M_1(bV)}{b}=M_0(V)$.
Consider the limiting ODE
\begin{equation}\label{scale}
\dot{V}(t)=M_0(V(t))-V(t).
\end{equation}
Observe that the globally asymptotically stable equilibrium
of the ODE (\ref{scale}) is the origin. Also, notice that the ODE (\ref{scale}) is a scaled limit of
the ODE (\ref{original}). Boundedness of $V(.)$ follows \cite{borkar2000ode}.\par
Boundedness of iterates of $T$ follows from  (\ref{dual1}).
\end{proof}
The physical interpretation behind the proof is as follows. If the iterates of the value functions
become unbounded along a subsequence, then a scaled version of the original ODE follows the ODE approximately. Since we have shown that
the scaled ODE must globally asymptotically converge to the origin, the scaled ODE must return to the origin. Therefore,
the value function iterates must also move towards a bounded set. This ensures the stability of the value function iterates.
\begin{lemma}
$V_n - V^{T_n} \to 0 $ a.s., where $V^{T_n}$ is the value function of the states for $T=T_n$.
\end{lemma}
\begin{proof}
We know that the threshold is varied on a slower timescale than that of $V$. Therefore, the value function iterates treat the threshold
value as constant.
Therefore, $T$ iterations can be viewed as $T_{n+1}=T_n+\alpha(n)$, where $\alpha(n)=O(h(n))=o(g(n))$.
Thus, the limiting ODEs associated with value function and threshold iterates are $\dot{V}(t)=M_1(V(t))-V(t)$ and $\dot{T}(t)=0$, respectively.
Since $\dot{T}(t)=0$,  it is sufficient to consider the ODE $\dot{V}(t)=M_1(V(t))-V(t)$, for a fixed value of $T$. The rest of the proof
is similar to \cite{borkar2005actor}.
\end{proof}
The subsequent lemmas prove that the average reward under a threshold $T$ ($\sigma(T)$) is unimodal in $T$, and hence the threshold iterations $T_n$
converge to the optimal threshold $T^*$. Therefore, $(V_n,T_n)$
converges to $(V,T^*)$.
\begin{lemma}\label{sec:dec}
 $v_n(i+1)-v_n(i)$ is non-increasing in $n$.
\end{lemma}
\begin{proof}
 Proof is provided in Appendix \ref{appb}.
\end{proof}

\begin{lemma}\label{sec:uni}
  $\sigma(T)$ is unimodal in $T$.
\end{lemma}
\begin{proof}
 Proof is provided in Appendix \ref{appc}.
\end{proof}

\begin{lemma}
 The threshold iterates $T_n\to T^*$.
\end{lemma}
\begin{proof}
The limiting ODE for Equation (\ref{dual2}) is the gradient ascent
\begin{equation*}
 \dot{T}=\nabla \sigma(T),
\end{equation*}
with inward pointing gradient at $0, N$. Using Lemma \ref{sec:uni}, there does not exist any local maximum other than the global maximum $T^*$.
Therefore $T_n \to T^*$.
\end{proof}
\begin{remark}
 In general, in an MDP problem with a threshold structure, unimodality of average reward may not hold.
 In such cases, the threshold iterates may converge to a local maximum only.
\end{remark}

\section {Proof of Lemma \ref{sec:dec}}\label{appb}
 We need to prove that $D v_n(i)$ is non-increasing in $n$. We use induction. When $n=0$, $v_0(i)=0$ and $D v_0(i)=0$.
 We have,
 \begin{equation*}
 v_{n+1}(i)=p\max\{v_n(i),r+v_n(i+1)\}+(1-p)v_n((i-1)^{+}).
 \end{equation*}
 \begin{equation*}
 v'_{n+1}(i)=\max\{v_n(i),r+v_n(i+1)\}.
 \end{equation*}
 Let $v''_{n+1}(i)=v_n((i-1)^{+})$.
Then $Dv'_1(i)=0$
and $D v_1(i)\le D v_0(i)$. \par
Now, assume that the claim holds for any $n$, i.e., $D v_{n+1}(i)\le D v_n(i)$. We need to prove that
$D v_{n+2}(i)\le D v_{n+1}(i)$. It is easy to see that $D v''_{n+2}(i)\le D v''_{n+1}(i)$. Therefore, to complete the proof, we need
to prove that $D v'_{n+2}(i)\le D v'_{n+1}(i)$.
Let $a_0,a_1\in \lbrace A_1,A_2 \rbrace$ be the maximizing actions in states $i$ and $i+1$, respectively, at $(n+2)^{\rm th}$ iteration. Let
$b_0,b_1\in \lbrace A_1,A_2 \rbrace$ be the maximizing actions in states $i$ and $i+1$, respectively, at $(n+1)^{\rm th}$ iteration.
Now, it is impossible to have $a_1=A_2$ and $b_0=A_1$. If $b_0=A_1$, we have, $D v_n(i)\le -r$.
From Lemma \ref{lemma1}, we must have $D v_n(i+1)\le -r$. If $a_1=A_2$, we have
$D v_{n+1}(i+1)\ge -r$.  This contradicts the
inductive assumption.
Therefore, we consider three cases as follows. For a given value of $a_1$ and $b_0$, if the inequality holds
for any values of $a_0$ and $b_1$, then the inequality will hold for maximizing actions as well.\\
1) $a_1=b_0=A_1,$ then choose $a_0=b_1=A_1$. We have,
\begin{equation*}
 \begin{split}
  &D v'_{n+2}(i)-D v'_{n+1}(i)=D v'_{n+1}(i)-D v'_{n}(i)\le 0.
 \end{split}
\end{equation*}
2) If $a_1=b_0=A_2,$ then we choose $a_0=b_1=A_2$,
and the inequality satisfies similar to the previous case.\\
3) If $a_1=A_1$ and $b_0=A_2$, then we choose $a_0=A_2$ and $b_1=A_1$. 
\begin{equation*}
\begin{split}
&D v'_{n+2}(i)-D v'_{n+1}(i)=v_{n+1}(i+1)-r-v_{n+1}(i+1)-v_{n}(i+1)\\&
+r+v_{n}(i+1)=0.
\end{split}
\end{equation*}


Thus, we have, $D v_{n+2}(i)\le D v_{n+1}(i)$. 

\section{Proof of Lemma \ref{sec:uni}}\label{appc}
We know that if the optimal action in state $i$ is $A_1$, then $V(i+1)-V(i)\le -r$.
Since VIA converges to the policy with threshold $T^*$, $\exists N_0>0$ such that $\forall n\ge N_0$,
$v_n(i+1)-v_n(i)\le -r$ $ \forall i \ge T^*$ and $v_n(i+1)-v_n(i)\ge -r$ $\forall i \le T^*$.
Let $t_n, n\ge 1$ be the optimal threshold at $n^{\rm {th}}$ iteration of VIA. Symbolically,
$t_n=\min\{i\in \mathbb{N}_0: v_n(i+1)-v_n(i)\le -r\}$. If for no values of $i$, the inequality holds,
then $t_n$ is taken as $N$. Using Lemma \ref{sec:dec}, $t_n$ must monotonically decrease with $n$ and
$\lim \limits_{n\to \infty} t_n= T^*$.\par
Consider a modified problem where $A_1$ is not permitted in any state $i<\hat{T}$, for a given
threshold $\hat{T} (T^*<\hat{T}\le N)$. Lemma \ref{sec:dec} holds for this modified problem too. Let $n_{\hat{T}}$ be the first VIA
iteration where the threshold drops to $\hat{T}$. The value function iterates for the modified and the original problem are same
for $n<n_{\hat{T}}$ because $A_1$ is never chosen as the optimal action for $i<\hat{T}$ in the original problem in these iterations. Therefore, $n_{\hat{T}}$ must be finite
and the following inequality holds for both original and the modified problem after $n_{\hat{T}}$ iterations.
\begin{equation}\label{eq:ineq}
 v_{n}(\hat{T}+1)-v_{n}(\hat{T})\le -r.
\end{equation}
Using Lemma \ref{sec:dec}, Equation (\ref{eq:ineq}) holds $\forall n \ge n_{\hat{T}}$.
Therefore, in the considered modified problem, $t_n$ converges to $\hat{T}$. This implies that the threshold policy
with threshold $\hat{T}$ is better than that of $\hat{T}+1$. Since $\hat{T}$ can be chosen arbitrarily, average reward is
monotonically decreasing with $\hat{T}$, $\forall \hat{T}>T^*$.\par
Now, if we have $\sigma(T)\ge \sigma(T+1)$, we must have $T\ge T^*$. Therefore, $\sigma(T+1)\ge \sigma(T+2)$.
Thus, $\sigma(T)$ is unimodal in $T$.\\\\
\textbf{Acknowledgements}
Work of VSB is supported in part by a J. C. Bose
Fellowship and CEFIPRA grant for ``Machine Learning for
Network Analytics''.
\bibliographystyle{ACM-Reference-Format}
\bibliography{valuetools}


\begin{thebibliography}{21}


\ifx \showCODEN    \undefined \def \showCODEN     #1{\unskip}     \fi
\ifx \showDOI      \undefined \def \showDOI       #1{#1}\fi
\ifx \showISBNx    \undefined \def \showISBNx     #1{\unskip}     \fi
\ifx \showISBNxiii \undefined \def \showISBNxiii  #1{\unskip}     \fi
\ifx \showISSN     \undefined \def \showISSN      #1{\unskip}     \fi
\ifx \showLCCN     \undefined \def \showLCCN      #1{\unskip}     \fi
\ifx \shownote     \undefined \def \shownote      #1{#1}          \fi
\ifx \showarticletitle \undefined \def \showarticletitle #1{#1}   \fi
\ifx \showURL      \undefined \def \showURL       {\relax}        \fi
\providecommand\bibfield[2]{#2}
\providecommand\bibinfo[2]{#2}
\providecommand\natexlab[1]{#1}
\providecommand\showeprint[2][]{arXiv:#2}

\bibitem[\protect\citeauthoryear{Abounadi, Bertsekas, and Borkar}{Abounadi
  et~al\mbox{.}}{2001}]%
        {abounadi2001learning}
\bibfield{author}{\bibinfo{person}{Jinane Abounadi}, \bibinfo{person}{D
  Bertsekas}, {and} \bibinfo{person}{Vivek~S Borkar}.}
  \bibinfo{year}{2001}\natexlab{}.
\newblock \showarticletitle{Learning algorithms for {M}arkov decision processes
  with average cost}.
\newblock \bibinfo{journal}{\emph{SIAM Journal on Control and Optimization}}
  \bibinfo{volume}{40}, \bibinfo{number}{3} (\bibinfo{year}{2001}),
  \bibinfo{pages}{681--698}.
\newblock


\bibitem[\protect\citeauthoryear{Agarwal, Borkar, and Karandikar}{Agarwal
  et~al\mbox{.}}{2008}]%
        {agarwal2008structural}
\bibfield{author}{\bibinfo{person}{Mukul Agarwal}, \bibinfo{person}{Vivek~S
  Borkar}, {and} \bibinfo{person}{Abhay Karandikar}.}
  \bibinfo{year}{2008}\natexlab{}.
\newblock \showarticletitle{Structural properties of optimal transmission
  policies over a randomly varying channel}.
\newblock \bibinfo{journal}{\emph{IEEE Trans. Automat. Control}}
  \bibinfo{volume}{53}, \bibinfo{number}{6} (\bibinfo{year}{2008}),
  \bibinfo{pages}{1476--1491}.
\newblock


\bibitem[\protect\citeauthoryear{Borkar}{Borkar}{2005}]%
        {borkar2005actor}
\bibfield{author}{\bibinfo{person}{Vivek~S Borkar}.}
  \bibinfo{year}{2005}\natexlab{}.
\newblock \showarticletitle{An actor-critic algorithm for constrained {M}arkov
  decision processes}.
\newblock \bibinfo{journal}{\emph{Systems \& control letters}}
  \bibinfo{volume}{54}, \bibinfo{number}{3} (\bibinfo{year}{2005}),
  \bibinfo{pages}{207--213}.
\newblock


\bibitem[\protect\citeauthoryear{Borkar}{Borkar}{2008}]%
        {borkar2008stochastic}
\bibfield{author}{\bibinfo{person}{Vivek~S Borkar}.}
  \bibinfo{year}{2008}\natexlab{}.
\newblock \bibinfo{booktitle}{\emph{Stochastic approximation: A dynamical
  systems viewpoint}}.
\newblock \bibinfo{publisher}{Cambridge University Press}.
\newblock


\bibitem[\protect\citeauthoryear{Borkar and Chadha}{Borkar and Chadha}{2018}]%
        {borkar2018reinforcement}
\bibfield{author}{\bibinfo{person}{Vivek~S Borkar} {and} \bibinfo{person}{Karan
  Chadha}.} \bibinfo{year}{2018}\natexlab{}.
\newblock \showarticletitle{A reinforcement learning algorithm for restless
  bandits}. In \bibinfo{booktitle}{\emph{Indian Control Conference (ICC).
  2018}}. IEEE, \bibinfo{pages}{89--94}.
\newblock


\bibitem[\protect\citeauthoryear{Borkar and Meyn}{Borkar and Meyn}{2000}]%
        {borkar2000ode}
\bibfield{author}{\bibinfo{person}{Vivek~S Borkar} {and}
  \bibinfo{person}{Sean~P Meyn}.} \bibinfo{year}{2000}\natexlab{}.
\newblock \showarticletitle{The {ODE} method for convergence of stochastic
  approximation and reinforcement learning}.
\newblock \bibinfo{journal}{\emph{SIAM Journal on Control and Optimization}}
  \bibinfo{volume}{38}, \bibinfo{number}{2} (\bibinfo{year}{2000}),
  \bibinfo{pages}{447--469}.
\newblock


\bibitem[\protect\citeauthoryear{Brouns and Van Der~Wal}{Brouns and Van
  Der~Wal}{2006}]%
        {brouns2006optimal}
\bibfield{author}{\bibinfo{person}{Gido~AJF Brouns} {and} \bibinfo{person}{Jan
  Van Der~Wal}.} \bibinfo{year}{2006}\natexlab{}.
\newblock \showarticletitle{Optimal threshold policies in a two-class
  preemptive priority queue with admission and termination control}.
\newblock \bibinfo{journal}{\emph{Queueing Systems}} \bibinfo{volume}{54},
  \bibinfo{number}{1} (\bibinfo{year}{2006}), \bibinfo{pages}{21--33}.
\newblock


\bibitem[\protect\citeauthoryear{Fu and van~der Schaar}{Fu and van~der
  Schaar}{2012}]%
        {fu2012structure}
\bibfield{author}{\bibinfo{person}{Fangwen Fu} {and} \bibinfo{person}{Mihaela
  van~der Schaar}.} \bibinfo{year}{2012}\natexlab{}.
\newblock \showarticletitle{Structure-aware stochastic control for transmission
  scheduling}.
\newblock \bibinfo{journal}{\emph{IEEE Transactions on Vehicular Technology}}
  \bibinfo{volume}{61}, \bibinfo{number}{9} (\bibinfo{year}{2012}),
  \bibinfo{pages}{3931--3945}.
\newblock


\bibitem[\protect\citeauthoryear{Konda and Borkar}{Konda and Borkar}{1999}]%
        {konda1999actor}
\bibfield{author}{\bibinfo{person}{Vijaymohan~R Konda} {and}
  \bibinfo{person}{Vivek~S Borkar}.} \bibinfo{year}{1999}\natexlab{}.
\newblock \showarticletitle{Actor-Critic-Type Learning Algorithms for {M}arkov
  Decision Processes}.
\newblock \bibinfo{journal}{\emph{SIAM Journal on control and Optimization}}
  \bibinfo{volume}{38}, \bibinfo{number}{1} (\bibinfo{year}{1999}),
  \bibinfo{pages}{94--123}.
\newblock


\bibitem[\protect\citeauthoryear{Koole}{Koole}{1998}]%
        {koole1998structural}
\bibfield{author}{\bibinfo{person}{Ger Koole}.}
  \bibinfo{year}{1998}\natexlab{}.
\newblock \showarticletitle{Structural results for the control of queueing
  systems using event-based dynamic programming}.
\newblock \bibinfo{journal}{\emph{Queueing systems}} \bibinfo{volume}{30},
  \bibinfo{number}{3-4} (\bibinfo{year}{1998}), \bibinfo{pages}{323--339}.
\newblock


\bibitem[\protect\citeauthoryear{Kunnumkal and Topaloglu}{Kunnumkal and
  Topaloglu}{2008}]%
        {kunnumkal2008exploiting}
\bibfield{author}{\bibinfo{person}{Sumit Kunnumkal} {and}
  \bibinfo{person}{Huseyin Topaloglu}.} \bibinfo{year}{2008}\natexlab{}.
\newblock \showarticletitle{Exploiting the structural properties of the
  underlying Markov decision problem in the Q-learning algorithm}.
\newblock \bibinfo{journal}{\emph{INFORMS Journal on Computing}}
  \bibinfo{volume}{20}, \bibinfo{number}{2} (\bibinfo{year}{2008}),
  \bibinfo{pages}{288--301}.
\newblock


\bibitem[\protect\citeauthoryear{Marbach and Tsitsiklis}{Marbach and
  Tsitsiklis}{2001}]%
        {marbach2001simulation}
\bibfield{author}{\bibinfo{person}{Peter Marbach} {and} \bibinfo{person}{John~N
  Tsitsiklis}.} \bibinfo{year}{2001}\natexlab{}.
\newblock \showarticletitle{Simulation-based optimization of {M}arkov reward
  processes}.
\newblock \bibinfo{journal}{\emph{IEEE Trans. Automat. Control}}
  \bibinfo{volume}{46}, \bibinfo{number}{2} (\bibinfo{year}{2001}),
  \bibinfo{pages}{191--209}.
\newblock


\bibitem[\protect\citeauthoryear{Ngo and Krishnamurthy}{Ngo and
  Krishnamurthy}{2009}]%
        {ngo2009optimality}
\bibfield{author}{\bibinfo{person}{Minh~Hanh Ngo} {and} \bibinfo{person}{Vikram
  Krishnamurthy}.} \bibinfo{year}{2009}\natexlab{}.
\newblock \showarticletitle{Optimality of threshold policies for transmission
  scheduling in correlated fading channels}.
\newblock \bibinfo{journal}{\emph{IEEE Transactions on Communications}}
  \bibinfo{volume}{57}, \bibinfo{number}{8} (\bibinfo{year}{2009}).
\newblock


\bibitem[\protect\citeauthoryear{Ngo and Krishnamurthy}{Ngo and
  Krishnamurthy}{2010}]%
        {ngo2010monotonicity}
\bibfield{author}{\bibinfo{person}{Minh~Hanh Ngo} {and} \bibinfo{person}{Vikram
  Krishnamurthy}.} \bibinfo{year}{2010}\natexlab{}.
\newblock \showarticletitle{Monotonicity of constrained optimal transmission
  policies in correlated fading channels with ARQ.}
\newblock \bibinfo{journal}{\emph{IEEE Trans. Signal Processing}}
  \bibinfo{volume}{58}, \bibinfo{number}{1} (\bibinfo{year}{2010}),
  \bibinfo{pages}{438--451}.
\newblock


\bibitem[\protect\citeauthoryear{Powell}{Powell}{2007}]%
        {powell2007approximate}
\bibfield{author}{\bibinfo{person}{Warren~B Powell}.}
  \bibinfo{year}{2007}\natexlab{}.
\newblock \bibinfo{booktitle}{\emph{Approximate Dynamic Programming: Solving
  the curses of dimensionality}}. Vol.~\bibinfo{volume}{703}.
\newblock \bibinfo{publisher}{John Wiley \& Sons}.
\newblock


\bibitem[\protect\citeauthoryear{Puterman}{Puterman}{2014}]%
        {puterman2014markov}
\bibfield{author}{\bibinfo{person}{Martin~L Puterman}.}
  \bibinfo{year}{2014}\natexlab{}.
\newblock \bibinfo{booktitle}{\emph{Markov decision processes: discrete
  stochastic dynamic programming}}.
\newblock \bibinfo{publisher}{John Wiley \& Sons}.
\newblock


\bibitem[\protect\citeauthoryear{Salodkar, Bhorkar, Karandikar, and
  Borkar}{Salodkar et~al\mbox{.}}{2008}]%
        {salodkar2008line}
\bibfield{author}{\bibinfo{person}{Nitin Salodkar}, \bibinfo{person}{Abhijeet
  Bhorkar}, \bibinfo{person}{Abhay Karandikar}, {and} \bibinfo{person}{Vivek~S
  Borkar}.} \bibinfo{year}{2008}\natexlab{}.
\newblock \showarticletitle{An on-line learning algorithm for energy efficient
  delay constrained scheduling over a fading channel}.
\newblock \bibinfo{journal}{\emph{IEEE Journal on Selected Areas in
  Communications}} \bibinfo{volume}{26}, \bibinfo{number}{4}
  (\bibinfo{year}{2008}), \bibinfo{pages}{732--742}.
\newblock


\bibitem[\protect\citeauthoryear{Sinha and Chaporkar}{Sinha and
  Chaporkar}{2012}]%
        {sinha2012optimal}
\bibfield{author}{\bibinfo{person}{Abhinav Sinha} {and}
  \bibinfo{person}{Prasanna Chaporkar}.} \bibinfo{year}{2012}\natexlab{}.
\newblock \showarticletitle{Optimal power allocation for a renewable energy
  source}. In \bibinfo{booktitle}{\emph{2012 National Conference on
  Communications (NCC)}}. IEEE, \bibinfo{pages}{1--5}.
\newblock


\bibitem[\protect\citeauthoryear{Sutton and Barto}{Sutton and Barto}{1998}]%
        {sutton1998reinforcement}
\bibfield{author}{\bibinfo{person}{Richard~S Sutton} {and}
  \bibinfo{person}{Andrew~G Barto}.} \bibinfo{year}{1998}\natexlab{}.
\newblock \bibinfo{booktitle}{\emph{Reinforcement learning: An introduction}}.
\newblock \bibinfo{publisher}{MIT press Cambridge}.
\newblock


\bibitem[\protect\citeauthoryear{Sutton, McAllester, Singh, and Mansour}{Sutton
  et~al\mbox{.}}{2000}]%
        {sutton2000policy}
\bibfield{author}{\bibinfo{person}{Richard~S Sutton}, \bibinfo{person}{David~A
  McAllester}, \bibinfo{person}{Satinder~P Singh}, {and}
  \bibinfo{person}{Yishay Mansour}.} \bibinfo{year}{2000}\natexlab{}.
\newblock \showarticletitle{Policy gradient methods for reinforcement learning
  with function approximation}. In \bibinfo{booktitle}{\emph{Advances in neural
  information processing systems}}. \bibinfo{pages}{1057--1063}.
\newblock


\bibitem[\protect\citeauthoryear{Watkins and Dayan}{Watkins and Dayan}{1992}]%
        {watkins1992q}
\bibfield{author}{\bibinfo{person}{Christopher~JCH Watkins} {and}
  \bibinfo{person}{Peter Dayan}.} \bibinfo{year}{1992}\natexlab{}.
\newblock \showarticletitle{Q-learning}.
\newblock \bibinfo{journal}{\emph{Machine learning}} \bibinfo{volume}{8},
  \bibinfo{number}{3-4} (\bibinfo{year}{1992}), \bibinfo{pages}{279--292}.
\newblock


\end{thebibliography}

\end{document}